%% file: ordering.tex
\documentclass[conference]{IEEEtran}
\IEEEoverridecommandlockouts
\usepackage[dvips]{graphicx}
\usepackage{amssymb}
\usepackage{graphicx}
\usepackage[cmex10]{amsmath}
\usepackage{amsthm}
\usepackage{algorithm}
\usepackage[noend]{algpseudocode}
\usepackage{array}
\usepackage{eqparbox}
\usepackage{diagbox}
\usepackage[tight,footnotesize]{subfigure}
\usepackage{float}
\usepackage{url}
\usepackage{paralist}
\usepackage{multirow}
\usepackage{wrapfig}
\usepackage{color}
\usepackage{tikz}
\usetikzlibrary{trees}
\usetikzlibrary{shapes.geometric, arrows}
\usepackage{cite}
\usepackage{atbegshi}% http://ctan.org/pkg/atbegshi
\AtBeginDocument{\AtBeginShipoutNext{\AtBeginShipoutDiscard}}
 
\newtheorem{proposition}{Proposition}

\newcommand{\RN}[1]{%
	\textup{\uppercase\expandafter{\romannumeral#1}}%
}
\begin{document}

	\title{Decision Tree Design for Classification in Crowdsourcing Systems}	
	\author{\IEEEauthorblockN{Baocheng Geng, Qunwei Li, Pramod K. Varshney\\}
		\IEEEauthorblockA{Department of Electrical Engineering and Computer Science, Syracuse University, NY, 13244, USA\\
			%line 2: name of organization, acronyms acceptable\\
			%line 3: City, Country\\
			%:
			\{bageng, qli33, varshney\}@syr.edu}}
	%\author{Nianxia Cao}
	\thanks{This work was supported in part
by NSF under Grant ENG 1609916, and in part by AFOSR under Grants FA9550-17-0313 and FA9550-16-1-0077.}
	\maketitle
		
	\begin{abstract}
%  In a crowdsourcing platform, we look at the test ordering issue when solving classification problems. Specifically, an object passes through a series of binary tests before it gets classified to a class. Since there are error probability with each test, our goal is to minimize the average miss classification rate by designing the sequence of tests that objects will pass through. Using entropy as a criterion, we build the testing algorithm level by level using a heuristic way. Further, worker assignment problem is studied when we are capable of recruiting a certain amount of person to improve the performance. Simulations show the efficiency of our proposed method.
In this paper, we present a novel sequential paradigm for classification in crowdsourcing systems.  Considering that workers are unreliable and they perform the tests with errors, we study the construction of decision trees so as to minimize the probability of mis-classification. By exploiting the connection between probability of mis-classification and entropy at each level of the decision tree, we propose two algorithms for decision tree design. Furthermore, the worker assignment problem is studied when workers can be assigned to different tests of the decision tree  to provide a trade-off between classification cost and resulting error performance. Numerical results are presented for illustration.
	\end{abstract}
	
% 	\begin{IEEEkeywords}
% 		Crowdsourcing, classification, task ordering, majority voting.
% 	\end{IEEEkeywords}

%Finally, we incorporate the costs of test in the construction of the testing algorithm.

%\input{FormulationHints_Qunwei}

\input{FormulationHints_Baocheng}

%++++++++++++++++++++++++++++++++++++++++
% References section will be created automatically 
% with inclusion of "thebibliography" environment
% as it shown below. See text starting with line
% \begin{thebibliography}{99}
% Note: with this approach it is YOUR responsibility to put them in order
% of appearance.

\bibliography{refer}
\bibliographystyle{IEEEtran}
\begin{appendices} \section{Proof of Proposition 1}
	\begin{proof}
		The cumulative distribution function of random variable $x$ from a binomial distribution with expected success probability $p_s$ can be expressed using the regularized incomplete beta function:
		\begin{equation*}
			F(m,n;p_s)= Pr(x\leq m) = I_{(1-p_s)}(n-m,m+1)
		\end{equation*}
		where
			$I_r(a,b)=\frac{B(r;a,b)}{B(1;a,b)}$
		and
$
			B(r;a,b) = \int\limits_0^r t^{a-1}(1-t)^{b-1}dt
$.
		
		In majority voting with $n=2k+1$ workers, each worker has an expected probability of success $1-p_e$, the probability of miss classification at FC can be expressed as: 
		\begin{align*}
			f_e(k) &= Pr(x\leq k) = F(k,2k+1,1-p_e)\\
                   &=I_{p_e}(k+1,k+1)
		\end{align*}
		
		Note that now $k$ can be any real value $k\geq 0$.
 		Taking partial derivative of $I_{p_e}(j,j)$ with respect to $j$ yields
 		\begin{align}
 			&\frac{dI_{p_e}(j,j)}{dj}\nonumber \\
             &=B^{-2}(1;j,j)\times \int\limits_{0}^{p_e}\int\limits_{0}^{1}(t-t^2)^j(s-s^2)^j\ln\frac{t-t^2}{s-s^2}dsdt \label{step1}\\
 			&=B^{-2}(1;j,j)\times \int\limits_{0}^{p_e}\int\limits_{p_e}^{1-p_e}(t-t^2)^j(s-s^2)^j\ln\frac{t-t^2}{s-s^2}dsdt \label{step2}
 		\end{align}
 		From \eqref{step1} to \eqref{step2}, we use the symmetry of $t-t^2$ with respect to $0.5$, and the fact $p_e<0.5$. Finally, notice that $s-s^2>t-t^2>0$ in the interval $s\in (p_e,1-p_e)$, and $t\in(0,p_e)$, thus $\ln\frac{t-t^2}{s-s^2}<0$ and $\frac{dI_{p_e}(j,j)}{dJ}$ is strictly negative. Since $j=k+1$, it follows that $f_e(k)$ is decreasing with respect to $k$.
         Besides, as $j$ increases, the magnitude of $\frac{dI_{p_e}(j,j)}{dj}$ strictly decreases because $|t-t^2|\leq1$ and $|s-s^2|\leq 1$. Thus, the magnitude of derivative decreases as $k$ increases.

 	\end{proof}
 \end{appendices}

\end{document}

%% file: FormulationHints_Baocheng.tex
\section{Introduction}
In recent work on classification in crowdsourcing systems, complex questions are often replaced by a set of simpler binary questions (microtasks) to enhance classification performance \cite{7747496, li2017optimal, 6784318,Li:2017:CRC:3055601.3055607}. This is especially helpful in situations where crowd workers lack expertise for responding to complex questions directly.  Each worker is given the entire set of questions in a batch mode and the workers provide their responses in the form of a vector. These binary questions can be posted as ``microtasks'' on crowdsourcing platforms like Amazon Mechanical Turk\cite{buhrmester2011amazon}. To improve classification performance in crowdsourcing systems, most of the works in the literature focus on enhancing the quality of individual tests, by designing fusion rules to combine decisions from heterogeneous workers \cite{dekel2009vox,ipeirotis2010quality,7747496, li2017optimal, 6784318,Li:2017:CRC:3055601.3055607}, and by investigating the assignment of different tests to different workers depending upon their skill level \cite{ho2013adaptive,roy2015task}. These problems have also been extended to budget-constrained environments to improve classification performance\cite{liu2012cdas,ho2012online,karger2011iterative}. 

In this paper, we present a new paradigm for classification in crowdsourcing systems in which binary questions (micro-tasks) are asked in a sequential manner. This novel sequential paradigm in terms of a decision tree has not been considered in the literature.  This paradigm provides the opportunity to order the sequence of tests for more efficient classification by reducing the number of questions asked on an average. Furthermore, we can obtain a trade-off in terms of cost (number of questions asked) and performance by performing task assignment and using only a subset of workers per node of the decision tree.  {Best performance with the decision tree paradigm can be achieved when all workers respond to every test in the decision tree.  However, as shown in this paper, the performance with the proposed worker assignment, where each worker only responds to one test as opposed to all the tests in the tree, is comparably when the number of workers is large.}

\textbf{Related work: }
Information theoretic methods have been used to construct efficient decision trees \cite{mitchell1997machine,hartmann1982application}. Classical algorithms utilize a top-down tree structure, such as ID3, C4.5, and CART \cite{quinlan1986induction,quinlan2014c4,breiman2017classification}. They categorize the objects at each node(test) into tree branches until a leaf is reached, and objects at this leaf are considered to belong to the same class. At each node, these algorithms search for a thresholding-based test on a certain attribute, such that the test can categorize the objects. ID3 and C4.5 construct the decision tree by maximizing the information gain at each node, which is defined as reduction in entropy. In CART, Gini impurity is minimized during test selection at each node.

The first strong assumption in traditional algorithms is that all the tests are error-free in determining whether or not an attribute exceeds the threshold.
In practical crowdsourcing systems, however, due to the noise in observing or measuring the attribute as well as human limitations, there exist errors and uncertainties when workers perform the tests. For objects belonging to different classes, the error probability corresponding to a specific test could also be different.
Existing algorithms do not address the concern that error probabilities of the tests play an important role in the design of decision trees.

Another limitation of these algorithms is the assumption of completely known information of object attributes to compute information gain and Gini impurity, i.e., probability $p(c_j|c_i,t)$  at node $t$, $c_j,c_i \in \mathcal C$, where $c_i$ is the correct class and $c_j$ is the result of the test. Even though some algorithms \cite{quinlan2014c4} can handle missing attributes information, they simply discard the missing attributes and use the remaining ones for decision tree construction. In the process of decision tree construction, they need to decide not only which attribute to use, but also the optimal threshold. The run time complexity goes up to $O(XY^2)$, where $X$ is the number of objects and $Y$ is number of attributes \cite{lim2000comparison}.
However, in practical crowdsourcing applications, we might not have the complete information $p(c_j|c_i,t)$. What we have are a limited number of tests (binary questions), and the corresponding test results. These above limitations of existing literature motivate the research results presented in this paper.

\textbf{Major contributions: }
Instead of assuming that each test in a decision tree is perfect, we consider the fact that there may be errors when tests are performed and develop an efficient algorithm to construct decision trees for the imperfect test scenario. The resulting tree is applicable to many practical problems including to classification performed by unreliable crowdsourcing workers. In our algorithm, the decision tree is constructed by utilizing a given set of tests, where each test gives a binary result $0$ or $1$ depending on which class the object belongs to. We do not assume a complete knowledge of $p(c_j|c_i,t)$. We provide performance guarantees in terms of the upper bound on probability of mis-classification (or the lower bound on probability of correct classification). {The time complexity of our algorithm is polynomial of $M$, which is the number of tests. Since $M$ is usually much smaller than $X$ and $Y$, our complexity is reduced significantly compared to other methods, e.g., the one proposed in \cite{lim2000comparison}.} After the decision tree is constructed, we employ it for classification via crowdsourcing. To reduce cost in terms of the number of questions asked while maintaining low probability of mis-classification,  we further develop an algorithm to efficiently assign workers to different tests, to obtain a trade-off between the probability of mis-classification and the cost of crowdsourcing.

\section{Decision Tree in Crowdsourcing}

\subsection{System Model}
Consider a classification problem to be solved via crowdsourcing. Suppose there is a set of objects $\mathcal O$, and each object within the set needs to be classified to a class $c_i \in \mathcal{C}$, $i\in \{1,2,\dots,N\}$. The prior probability that an object in $\mathcal O$ belongs to $c_i$ is denoted as $p(c_i)$. An unknown object passes through a series of simple tests (nodes in the decision tree) until it reaches a leaf node and gets classified. We consider that each test $T_m\in \{T_1,T_2,\dots,T_M\}$ provides a  binary output for a subset of $\mathcal O$, thus partitioning the subset of input objects into two output subsets. If an object belonging to $c_i$ gets mis-categorized at test $T_m$, a misclassification will happen in the end and this corresponding error probability is demoted by $p_{i,m}$. Table \ref{my-label} gives an example of test statistics and Fig. \ref{fig1} gives two possible testing algorithms. As indicated by Table \ref{my-label}, tests $\{T_i\}_{i=1}^4$ can bifurcate the entire set $\mathcal O$ and $T_5$ can only bifurcate a subset of objects belonging to the classes $\{c_1,c_2,c_3,c_5\}$. Assuming that all the tests have the same error probability $p_{i,m}=0.05$, the final misclassification probabilities in Fig. \ref{fig1}(a) and Fig. \ref{fig1}(b) are $0.068$ and $0.05$ respectively. Thus, we can see that even though the same set of tests are employed, different decision tree structures (ordering of tests) have different  probabilities of mis-classification. Our goal is to build a decision tree that minimizes the mis-classification probability.
\begin{table}[t]
	\centering
	\caption{Decision Model}
	\label{my-label}
	\begin{tabular}{|l|l|l|l|l|l|}
		\hline
		\diagbox{Test}{Class} & $c_1$ & $c_2$ & $c_3$ & $c_4$ & $c_5$ \\ \hline
		$p(c_i)$ & 0.20 & 0.05 & 0.10 & 0.60 & 0.05 \\ \hline
		$T_1$ & 0 & 0 & 0 & 1 & 0 \\ \hline
		$T_2$ & 1 & 0 & 0 & 1 & 1 \\ \hline
		$T_3$ & 0 & 1 & 0 & 0 & 1 \\ \hline
		$T_4$ & 0 & 1 & 0 & 1 & 1 \\ \hline
		$T_5$ & 0 & 1 & 1 & - & 1 \\ \hline
	\end{tabular}
\end{table}

\begin{figure}
	\centering
	\subfigure[Algorithm 1]  
	{
		\begin{tikzpicture}[level distance=0.8cm,
		level 2/.style={sibling distance=2cm},
		level 3/.style={sibling distance=1cm}]
		\tikzstyle{first}=[rectangle,draw]
		\tikzstyle{second}=[circle,draw]
		
		\node [first] {$c_1,c_2,c_3,c_4,c_5$} [->]
		child{node [second] {$T_3$}
			child { node [first]{$c_1,c_3,c_4$}child{
					node [second]{$T_4$} 
					child{node [first]{$c_1,c_3$}
						child { node [second]{$T_5$} 
							child{node[first]{$c_1$}}
							child{node[first]{$c_3$}} 
					}}
					child { node [first]{$c_4$} }	
			}}
			child {
				node [first] {$c_2,c_5$}
                child { node [second] {$T_2$}
				child {  node [first]{$c_2$} }
				child {  node [first]{$c_5$} }
			}}
		};
		\end{tikzpicture}
	}  
	% The only difference is here, where I have commented out an empty line.
	\hspace{0.2in}
	\subfigure[Algorithm 2]  
	{
		\begin{tikzpicture}[level distance=0.8cm,
		level 2/.style={sibling distance=2cm},
		level 3/.style={sibling distance=1cm},
		level 4/.style={sibling distance=1.5cm},]
		\tikzstyle{first}=[rectangle,draw]
		\tikzstyle{second}=[circle,draw]
		
		\node [first] {$c_1,c_2,c_3,c_4,c_5$} [->]
		child{node [second] {$T_1$}
			child { node [first]{$c_1,c_2,c_3,c_5$}
				child{
					node [second]{$T_5$}
					child{node [first]{$c_1$}}
					child { node [first]{$c_2,c_3,c_5$}
						child {node[second]{$T_3$}
							child{node [first]{$c_3$}}
							child{node [first]{$c_2,c_5$}
								child {node [second] {$T_2$}
									child{node [first]{$c_2$}}
									child{node [first] {$c_5$}}
				}}}}}
			}
			child {
				node [second] {$c_4$}
			}
		};
		\end{tikzpicture}
	}
	\caption{Two testing algorithms}
    \label{fig1}
    \vspace{-0.5cm}
\end{figure}
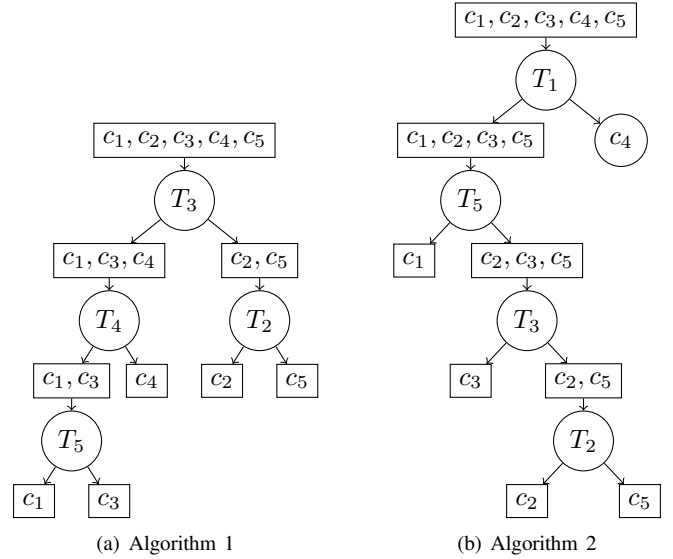
\begin{figure}
	\centering	
    \vspace{-0.3cm}
    \includegraphics[width=.34\textwidth]{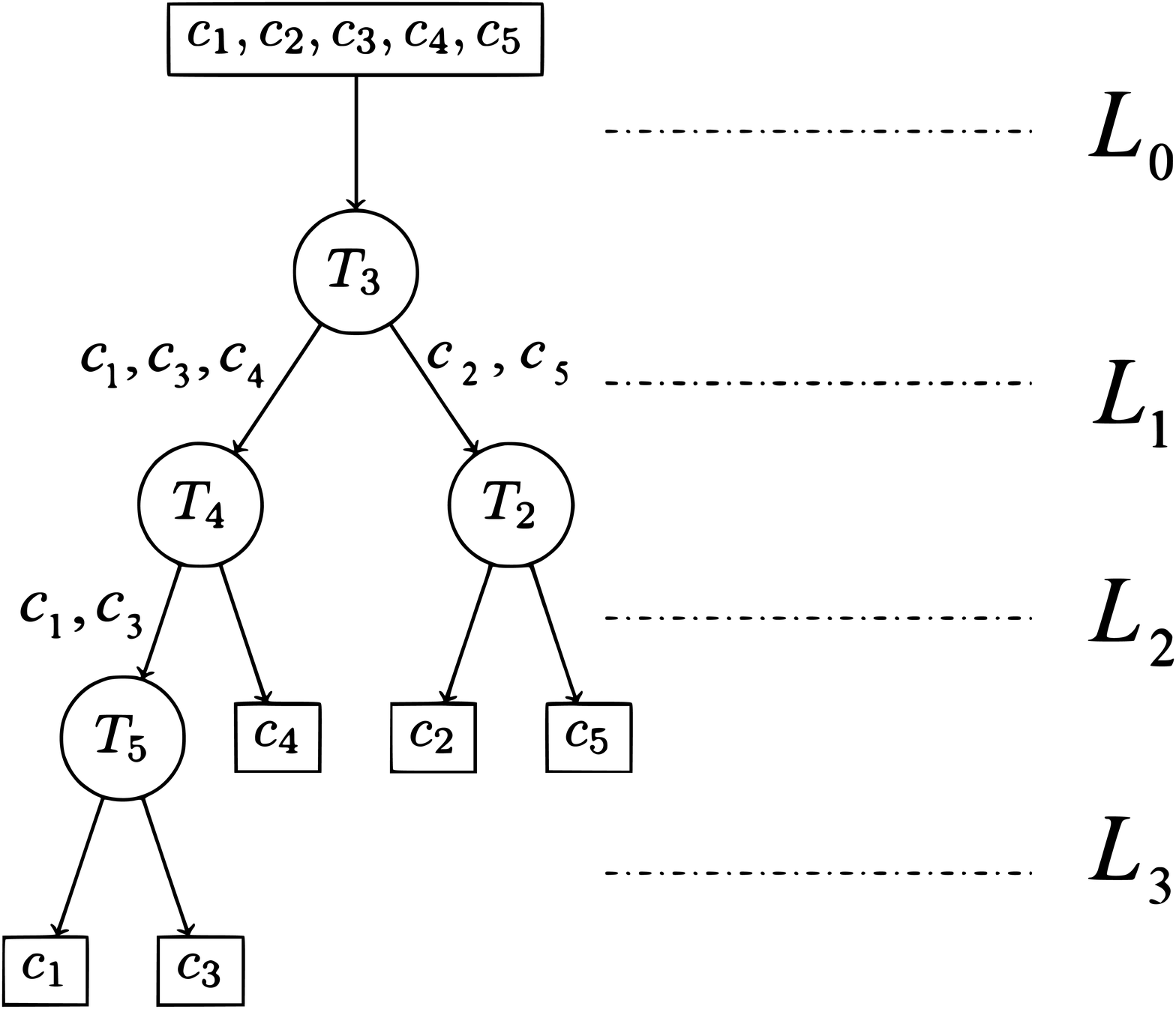}
	\caption{Illustration of Test Levels}
    \vspace{-0.6cm}
\end{figure}

Define the test level $L_d$, $d\in\{0,1,\dots,D\}$ as in Fig. 2, where $D$ is the depth of the tree structure. At each level $L_d$, define the partitions of classes induced by tests applied so far to be $\gamma_d=\{\gamma_{d1},\gamma_{d2},\dots,\gamma_{d|\gamma_d|}\}$, where $|\gamma_d|$ is the cardinality of the partition set $\gamma_d$ and it implies the degree of completion of the classification task. A larger $|\gamma_d|$ indicates closer to the completion of classification. In the example given in Fig. 2, we have:
\vspace{-0.15cm}
\begin{align*}
\gamma_0&=\{\{c_1,c_2,c_3,c_4,c_5\}\}\\
\gamma_1&=\{\{c_1,c_3,c_4\},\{c_2,c_5\}\}\\
\gamma_2&=\{\{c_1,c_3\},\{c_4\},\{c_2\},\{c_5\}\}\\
\gamma_3&=\left\{\{c_1\},\{c_2\},\{c_3\},\{c_4\},\{c_5\}\right\}
\end{align*}

Note $\gamma_D$ is where each class has been individually distinguished. Let $\Gamma_d$ denote the partition induced in $\gamma_d$.  We define the entropy at level $L_d$ as:
\begin{align}
H(L_d) = H(\Gamma_D|\Gamma_d)
&= -\sum\limits_{n,k}p(\gamma_{Dn},\gamma_{dk})\text{log}_2 p(\gamma_{Dn}|\gamma_{dk})\nonumber\\
&= \sum\limits_{k}p(\gamma_{dk})H(\Gamma_D|\Gamma_d=\gamma_{dk})
\end{align}
where $p(\gamma_{Dn},\gamma_{dk})$ is the joint probability of partitions $\gamma_{Dn}$ and $\gamma_{dk}$. Following this definition, $H(L_0)=-\sum\limits_{i=1}^{N}p(c_i){\log}_2p(c_i)$, and $H(L_D)=0$. The entropy at each level will be exploited in choosing the tests for the next level so as to minimize the final probability of mis-classification. 

\section{Proposed Decision Tree Design Algorithms}
In this section, we focus on algorithms for decision tree design.
We use two types of approximations, namely, addition approximation to minimize the upper bound  of mis-classification probabilities, and multiplication approximation to maximize the lower bound of correct classification probabilities, respectively. {Previous work \cite{hartmann1982application}, with the objective of minimizing the upper bound of test cost, e.g., memory, execution time, does not consider error in tests (noisy tests), which is very different from our paper.}

\subsection{Bounding the probability of mis-classification}
In a decision tree, the probability of mis-classification error is given by :
\begin{equation}
P_m=\sum\limits_{i=1}^{N}p(c_i)\left(1-\prod\limits_{d=1}^{D}(1-p^*_{i,d})\right),
\end{equation}
 where $p^*_{i,d}$ is the error probability associated with the unknown object belonging to $c_i$ as it traverses the node between levels $L_{d-1}$ and $L_{d}$. Note that if an object does not pass through a level of the tree, the corresponding error probability is $0$.
 
Typically, the error probability for each test is small. Otherwise, the corresponding test should be replaced by a better test to reduce the error probability. Since the error probability of each test is small, the probability of mis-classification can be approximated by dropping out the higher order terms as 
 \begin{equation}
 P_m\!\approx\!\sum\limits_{i=1}^{N}p(c_i)\sum\limits_{d=1}^{D}p^*_{i,d}.
 \end{equation}
Thus, we have the additive approximation
\begin{align}\label{add_approx}
P_m\!\approx\!\sum\limits_{i=1}^{N}p(c_i)\sum\limits_{d=1}^{D}p^*_{i,d}\!=\!\sum\limits_{d=1}^{D}\sum\limits_{i=1}^{N}p(c_i)p^*_{i,d}=\sum\limits_{d=1}^{D}g(d-1,d)  ,       
\end{align}
where $g(d-1,d)$ represents the error probability induced by tests between level $L_{d-1}$ and level $L_d$.
Recalling the definition of $H(L_d)$ in (1), and using the fact that $H(L_D)=0$, we can write:
\begin{align}\label{add_entro}
H\!(\!L_0\!)&\!=\!H\!(\!L_0\!)\!-\!H\!(\!L_1\!)\!+\!H\!(\!L_1\!)\!-\!H\!(\!L_2\!)\!+\!\dots\!+\!H\!(\!L_{D-1}\!)\!-\!H\!(\!L_D\!)\!\nonumber\\
&=\sum\limits_{d=1}^{D}\frac{H(L_{d-1})-H(L_d)}{g(d-1,d)}g(d-1,d)\nonumber\\
&=\sum\limits_{d=1}^{D}F^m(d-1,d)g(d-1,d),
\end{align}
where $F^m(d-1,d)=\frac{H(L_{d-1})-H(L_d)}{g(d-1,d)}$ is the metric we use for decision tree construction. It is the reduction in entropy from $L_{d-1}$ to $L_{d}$, divided by the error probability induced between these two levels. Essentially, it indicates the sensitivity to error for reducing uncertainty in decision tree design at a certain level. Define $F^m_{min}=\min\limits_{d=1,\dots,D}F^m(d-1,d)$, and $F^m_{max}=\max\limits_{d=1,\dots,D}F(d-1,d)$. Due to the fact that $H(L_{d-1})-H(L_d)\geq 0$, and $g(d-1,d)>0$, it follows that $F(d-1,d)\geq 0$. Substituting \eqref{add_approx} into \eqref{add_entro}, we can have
\begin{align*}
F^m_{min}P_m\leq H(L_0)\leq F^m_{max}P_m,
\end{align*}
which leads to
\begin{align*}
\frac{H(L_0)}{F^m_{max}}\leq P_m \leq \frac{H(L_0)}{F^m_{min}}.
\end{align*}
Since our goal is to minimize $P_m$, we are interested in minimizing the upper bound $\frac{H(L_0)}{F^m_{min}}$. Since $H_{L_0}$ is fixed, we need to maximize $F^m_{min}=\min\limits_{d=1,\dots,D}F^m(d-1,d)$. During the construction of testing algorithm, it is sufficient to maximize  each of $F^m(d-1,d)=\frac{H(L_{d-1})-H(L_d)}{g(d-1,d)}$, $d\in\{1,2,\dots,D\}$. When we construct the decision tree from level $d-1$ to level $d$, we select the tests that maximize the value $F^m(d-1,d)$, and the decision tree construction step ends when it reaches the $D$-th level.

\subsection{Bounding the probability of correct classification}
In this section, we focus on decision tree design to maximize the probability of correct classification, which can be written as
\begin{align}
P_c=\sum\limits_{i=1}^{N}p(c_i)\prod\limits_{d=1}^{D}(1-p^*_{i,d}).
\end{align}
Since the effect of higher order terms is negligible as typically they are small, we approximate $P_c$ as
\begin{align}
P_c&\approx\prod\limits_{d=1}^{D}\sum\limits_{i=1}^{N}p(c_i)(1-p^*_{i,d})=\prod\limits_{d=1}^{D} b(d-1,d),
\end{align}
%This approximation from (6) to (7) is by taking advantage of $p(c_i)\geq 0$ for $i\in\{1,2,\dots,N\}$, $\sum_{i=1}^{N}p(c_i)=1$, as well as $1-p_{i,d}^*\in(0,1)$, so that the differences of higher order terms in (6) and (7) average out. 
where $b(d-1,d)$ represents the probability of correct classification between level $L_{d-1}$ and $L_d$, $d\in \{1, 2, \dots, D\}$.

Then, we provide the entropy in the multiplicative form as
\begin{align}
H(L_0)+1&=\frac{H(L_0)+1}{H(L_1)+1}\times \frac{H(L_1)+1}{H(L_2)+1}	
\times ...\times \frac{H(L_{D-1})+1}{H(L_{D})+1}\nonumber\\
&=\prod _{d=1}^{D}\frac{\frac{H(L_{d-1})+1}{H(L_d)+1}}{b(d-1,d)}b(d-1,d)\nonumber\\
&=\prod_{d=1}^{D}F^c(d-1,d)b(d-1,d),
\end{align}
where $F^c(d-1,d)=\frac{\frac{H(L_{d-1})+1}{H(L_d)+1}}{b(d-1,d)}$ is the metric based on which we select tests. It is the generalized entropy ratio of levels $L_{d-1}$ and $L_{d}$, divided by the probability of correct classification between these two levels. Essentially, it indicates the degree of reduction in uncertainty when the test correctly bifurcates the objects. Define $F^c_{min}=\min\limits_{d=1,\dots,D}F^c(d-1,d)$, and $F^c_{max}=\max\limits_{d=1,\dots,D}F^c(d-1,d)$. Since  $F^c(d-1,d)\geq 0$, substitute (7) into (8) and we have
\begin{align*}
F^c_{min}P_c\leq H(L_0)+1\leq F^c_{max}P_c,
\end{align*}
which leads to
\begin{align*}
\frac{H(L_0)+1}{F^c_{max}}\leq P_c \leq \frac{H(L_0)+1}{F^c_{min}}.
\end{align*}
As we desire to maximize the probability of correct classification $P_c$, and thus are interested in maximizing its lower bound which is $\frac{H(L_0)+1}{F^c_{max}}$. Since $H(L_0)+1$ is fixed, we need to minimize $F^c_{max}$. During the construction of the decision tree, it is sufficient to select the tests that minimize the value $F^c(d-1,d)=\frac{\frac{H(L_{d-1})+1}{H(L_d)+1}}{b(d-1,d)}$ from level $d-1$ to level $d$. 

The additive approximation is obtained by discarding second to $D$th order terms of $p^*_{i,d}$, while the multiplicative approximation discards $D$th order of $p^*_{i,d}$. Thus, multiplicative approximation is more accurate than additive approximation. However, the tightness of the bounds on probability of correct classification in the multiplicative method depends on the metric $\frac{H(L_{d-1})+r}{H(L_d)+r}$. In this paper, we choose $r=1$, which might not be optimal.

In our simulation with the experimental setting as shown in Table 1, and when we assume that all the tests have the same error probability $p^*$, both methods give the same resulting decision tree (testing algorithm)  shown in Fig. 1(b). Fig. 3 shows the efficiency of the proposed decision tree design algorithm by comparing its probability of mis-classification (blue curve) with the case where tests are in a random order (red curve). As we can see from the figure, the performance is significantly improved with our methods, { and the improvement becomes more prominent as $p^*$ increases.}

\begin{figure}
	\centering
	\includegraphics[width=.4\textwidth]{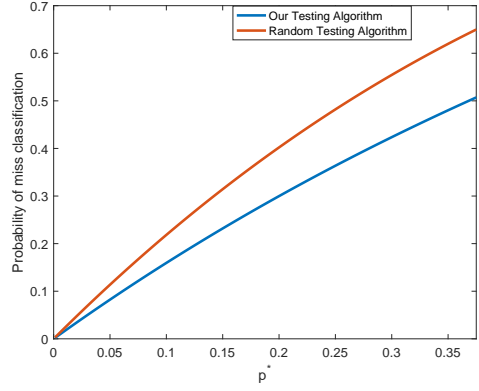}
     \vspace{-0.3cm}
	\caption{Probability of mis-classification when $p^*$ increases}
     \vspace{-0.5cm}
\end{figure}
 
\section{Worker Assignment}
After designing the decision tree, the next step is to assign the available crowd  workers to the nodes of the decision tree. The naive and the most costly approach will be to have all available workers answer questions corresponding to each node. This will mean that the number of questions answered will be $M_0N_0$, where $M_0$ is the number of nodes in the decision tree and $N_0$ is the total number of workers. The goal in this section is to investigate the trade-off between the saving in the number of questions answered (cost)  and the degradation in performance as well as to develop an efficient algorithm to assign subsets of workers to different nodes of the tree. In particular, each node must have at least one worker assigned to it; the goal is to find an algorithm to optimally distribute remaining crowd workers among the nodes of the decision tree. When subgroups of workers are assigned to perform different tests at individual nodes, the workers' local decisions are collected by a fusion center (FC). Majority voting is used in this paper for decision fusion for crowdsourcing. In a subgroup of workers with size $n=2k+1$, $k = 0,1,\dots$, each worker completes the same test that will produce binary results $0$ or $1$. The probability of error of the $i$th worker for the corresponding test is $p^i_e$. In majority rule, FC will follow the decisions made by the majority. That is, if at least $k+1$ workers declare $0$ to be the result, FC will decide $0$; otherwise, it will decide that $1$ is true. For a certain test, we provide the worker assignment scheme.
\begin{proposition}
	Suppose the expected probability of error of each worker for a certain test is $\mathbb{E}(p_e^i)=p_e$. When $p_e<0.5$, the probability of error at FC $f_e(k)$ is a decreasing function of $k$. The reduction in probability of error at FC decreases as well, as $k$ increases, i.e. $\left | f_e(k_1+1)-f_e(k_1) \right |\leq \left | f_e(k_2+1)-f_e(k_2) \right |$ for $k_1>k_2\geq 0$
\end{proposition} 
\begin{proof}
See Appendix A.
\end{proof}
{Under the assumption of Proposition 1: $\mathbb{E}(p_e^i)=p_e$, after we have constructed a testing algorithm, for example the one shown in Fig. 1(b), each test is assigned a randomly chosen worker.} After that, we assume that we have a group of additional $n = 2K$ workers available to reduce the error probabilities of one or more tests. Let $2k_m, m\in\{1,2,\dots,M\}$ be the number of workers assigned to test $T_m$. By doing so, we ensure that the number of workers performing test $T_m$ is odd, and $\sum\limits_{m=1}^{M}k_m =K$. We address the problem of how to assign these $2K$ workers to different tests, i.e., to determine the values of $\{k_1,k_2,\dots,k_M\}$, such that we can achieve minimum probability of mis-classification .

From the result of Proposition 1, as more workers are assigned to the same test, the rate of reduction in error probability decreases. Thus, we are encouraged to allocate two workers at a time to a certain test, to guarantee the odd number of workers for each test, and to ensure the maximal rate of reduction in error probability each time.{ Using the methods proposed in the previous section, we can construct the decision tree and find the level $d^\prime$ that has the minimal $F^m(d^\prime-1, d^\prime)$ or maximal $F^c(d^\prime-1,d^\prime)$ (both decision tree construction algorithms provide the same result). For the tests between level $L_{d^\prime-1}$ and $L_{d^\prime}$, we add two workers to the test that gives most increase in $F^m(d^\prime-1, d^\prime)$ or most decrease in $F^c(d^\prime-1,d^\prime)$.}  We provide the following worker assignment algorithm:
\begin{algorithm}
	\caption{Worker Assignment}\label{alg:euclid}
	\begin{algorithmic}[1]
		\Procedure{Assign $2K$ workers}{}
        \State Initialize $k_1=k_2=\dots=k_M=0$.
		\While{$n=2K>0$}
		\State Find $d'$.
		\State From $L_{d'-1}$ to $L_{d'}$, add two workers to                $T_m$ that gives most increase in $F^m(d^\prime-1, d^\prime)$, or most decrease in $F^c(d^\prime-1,d^\prime)$. 
        \State $k_m\leftarrow k_m+1$
		\State Update the value $F^m(d'-1,d')$, or $F^c(d'-1,d')$.
		\State $K \leftarrow K-1$.
		\EndWhile\label{euclidendwhile}
		\State \textbf{end}
		\EndProcedure
	\end{algorithmic}
\end{algorithm}

\begin{figure}
	\centering	\includegraphics[width=.4\textwidth]{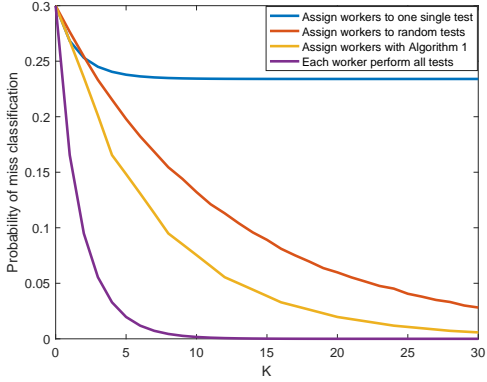}
    \vspace{-0.4cm}
	\caption{Probability of mis-classification as a function of $K$}
    \vspace{-0.5cm}
\end{figure}
In our simulations, each worker has an error probability of $p_e =0.2$ for all the tests, and Fig. 4 shows the probability of mis-classification when the number of workers increases. The blue curve represents the case when we assign all the workers to a single test (randomly chosen); the red curve indicates the scenario where each worker is randomly assigned to a test, and the yellow curves represents the proposed worker assignment rule associated with metric $F^m(d-1,d)$. We can see from the figure that one should not assign workers in a highly unbalanced fashion as is indicated by the blue curve. Random worker assignment achieves better performance, which is outperformed by our proposed method. The purple curve represents the scenario where each worker participates in all the tests in our decision tree. Though it has the best performance, the cost (number of tests answered by workers) induced is $M$ times higher, where $M$ is the number of tests. After $K>25$, we can see that our algorithm achieves comparable performance as the purple curve, however, with a significantly lower cost.

\section{Conclusion}
This work presented a novel sequential paradigm for crowdsourced classification and also addressed the test ordering problem. With limited knowledge of worker's reliability in performing imperfect tests, we provided a greedy decision tree design to minimize the probability of mis-classification. Two different methods were used to approximate the probabilities of mis-classification and correct classification. We also investigated the worker assignment problem, by studying the assignment of a limited number of workers to different tests. Numerical results showed the superiority of our testing algorithm, as well as the efficiency of the worker assignment strategy. While our greedy level-by-level decision tree construction only achieves local optimality, in future work, we will explore the possibility of obtaining globally optimal solutions.